\documentclass[conference,letter]{IEEEtran}
\usepackage[utf8]{inputenc}

\usepackage{amsmath, amsfonts, amssymb}
\usepackage{amsthm}
\usepackage{graphicx}
\usepackage{url}

\newcommand{\R}{\mathbb{R}}

\newcommand{\mc}[1]{\mathcal{#1}}

\newcommand{\im}{\operatorname{Im}}
\renewcommand{\P}{\mathbb{P}}

\newcommand{\E}{\mathbb{E}}
\newcommand{\KL}{\operatorname{KL}}

\renewcommand{\H}{\mathbb{H}}
\newcommand{\I}{\mathbb{I}}

\newtheorem{definition}{Definition}
\newtheorem{theorem}{Theorem}
\newtheorem{lemma}{Lemma}

\newtheorem{assumption}{Assumption}

\begin{document}
%
\title{Differentially Private Synthetic Data Generation via Lipschitz-Regularised Variational Autoencoders}
%
%
%

\author{\IEEEauthorblockN{Benedikt Groß\IEEEauthorrefmark{1}, Gerhard Wunder\IEEEauthorrefmark{1}}

\IEEEauthorblockA{\IEEEauthorrefmark{1}Department of Computer Science, 
Freie Universit\"at Berlin, Germany \\ \{benedikt.gross, g.wunder\}@fu-berlin.de
    }
}

\maketitle
\thispagestyle{empty}

\begin{abstract}

Synthetic data has been hailed as the silver bullet for privacy preserving data analysis. If a record is not real, then how could it violate a person's privacy? In addition, deep-learning based generative models are employed successfully to approximate complex high-dimensional distributions from data and draw realistic samples from this learned distribution. It is often overlooked though that generative models are prone to memorising many details of individual training records and often generate synthetic data that too closely resembles the underlying sensitive training data, hence violating strong privacy regulations as, e.g., encountered in health care. Differential privacy is the well-known state-of-the-art framework for guaranteeing protection of sensitive individuals' data, allowing aggregate statistics and even machine learning models to be released publicly without compromising privacy. The training mechanisms however often add too much noise during the training process, and thus severely compromise the utility of these private models. Even worse, the tight privacy budgets do not allow for many training epochs so that model quality cannot be properly controlled in practice. In this paper we explore an alternative approach for privately generating data that makes direct use of the inherent stochasticity in generative models, e.g., variational autoencoders. The main idea is to appropriately constrain the continuity modulus of the deep models \emph{instead} of adding another noise mechanism on top. For this approach, we derive mathematically rigorous privacy guarantees and illustrate its effectiveness with practical experiments.
\end{abstract}


%
\IEEEpeerreviewmaketitle

\section{Introduction}
\label{sec:introduction}

Differential privacy (DP) was introduced in \cite{dwork2006differential} to provide a mathematical definition of privacy that allows people to participate in the collection of sensitive data, without having to worry about their personal secrets being revealed by querying the collected information. While initial research focused mainly on protecting individual records against privacy leakage through queries on aggregate quantities of a database, more recent work is centered around differentially private training of neural networks. Problematic to the applications of differential privacy is that the actual computations involving sensitive data need to be done by the data holder. If these computations exceed simple database queries, whose answers can be secured by applying e.g. the Laplace mechanism, this requires the data holder to provide compute capabilities as well as the knowledge to assess the training routine's correctness with respect to privacy to prevent data leakage. Hence, enabling third parties to train advanced statistical models on private data calls for a new approach. A promising direction is to create synthetic data in a differentially private manner that can then be published without restricting its use in any way due to the immunity of differentially private mechanisms to post-processing. In \cite{stadler2021synthetic} it was demonstrated that creating synthetic data alone was not sufficient to protect privacy for all data records. Especially outliers in a dataset proved to be sensitive to membership inference (MIA) and feature inference attacks. While several differentially private generative models exist in the literature (e.g. \cite{ho2021dp, jordon2018pate, torkzadehmahani2019dp, long2021g}), most of these models rely on using either PATE \cite{DBLP:conf/iclr/PapernotAEGT17} or differentially private variants of stochastic gradient descend (DP-SGD) \cite{abadi2016deep} to train a model for the generation of DP synthetic data. Since DP-SGD relies on the addition of properly scaled noise to the training process, it is expected that imposing strict privacy guarantees severely hampers the utility of the generated data. PATE on the other hand needs a lot of training data and additional public data to train the student model on. Moreover, the training cost is high since many teacher models have to be trained. In this work we propose a variational autoencoder (VAE)\cite{kingma2013auto} with a constraint on the Lipschitz constant of its decoder as a privacy preserving generative model. Relying on the inherent stochasticity of the encoding process, the proposed method produces differentially private synthetic data without the drawbacks of training the VAE model with DP-SGD, which reduces sample quality through adding noise to the gradient updates and limits the number of epochs for which the model can be trained, since meaningful privacy budgets are quickly exhausted. 
Its privacy preserving properties are quantified in a theoretical result that builds upon an information theoretic view on the evidence lower bound (ELBO) objective \cite{hoffman2016elbo} and results on differentially private posterior sampling \cite{wang2015privacy}.
The proposed Lipschitz-constrained VAE is evaluated numerically in terms of privacy and utility of the generated data against a vanilla VAE and a VAE trained with DP-SGD, and is shown to offer a better privacy-utility trade-off.

The paper is organised as follows: First a brief overview is given on differential privacy and in particular differential privacy for synthetic data generation in Section \ref{sec:differential_privacy}. Then the proposed Lipschitz-constrained VAE (LVAE) is introduced in Section \ref{sec:vae} and its privacy preserving qualities are investigated theoretically in Section \ref{sec:privacy_analysis}. Numerical experiments are conducted in which the utility and privacy of the proposed model is benchmarked in Section \ref{sec:experiments}. The paper concludes by discussing results and open questions for future research in Section \ref{sec:conclusion}. 

\section{Differential Privacy}
\label{sec:differential_privacy}

We begin by briefly introducing the standard notions of differential privacy and then discuss the hypothesis test viewpoint of DP, which renders the obtained privacy guarantees interpretable. The section ends with a definition of DP for synthetic data generation and the membership inference attack that is used for the empirical privacy evaluation in the experiment section.
\begin{definition}{($\epsilon$-DP)}
\label{def:eddp}
A randomized algorithm $\mc M$ is called $\epsilon\geq 0$ Differentially private, if for all 
neighboring databases $D_0, D_1$ (i.e. differing only in one record) and for all subsets $S\subset \im(\mc M)$ it holds
\begin{equation}
    \label{eq:eps_dp}
    \P\left[ \mc M(D_0)\in S\right]\leq \exp(\epsilon)\cdot\P\left[\mc M(D_1)\in S\right]
\end{equation}
\end{definition}

Since pure $\epsilon$-DP is difficult to achieve for composite privacy mechanisms as used for training machine learning models, the following relaxation is most often employed in the literature.
\begin{definition}{($(\epsilon, \delta)$-DP)}
A randomized algorithm $\mc M$ is called $\epsilon\geq 0, \delta\geq 0$ differentially private, if for all 
neighboring databases $D_0, D_1$ and for all subsets $S\subset \im(\mc M)$ it holds
\begin{equation}
    \label{eq:eps_delta_dp}
    \P\left[ \mc M(D_0)\in S\right]\leq \exp(\epsilon)\cdot\P\left[\mc M(D_1)\in S\right]+\delta
\end{equation}
\end{definition}
Differential privacy is usually achieved by adding properly scaled noise coming either from a Laplace or Gaussian distribution to the query outputs. For the composition of mechanisms various theorems exist to bound the compined privacy cost. Currently, the tightest composition bounds are achieved by using a so-called moments accountant \cite{wang2019subsampled} or the Gaussian DP framework introduced below.

\subsection{Differential Privacy as Hypothesis Test Problem}

In \cite{kairouz2015composition}, a characterization of DP as a hypothesis test problem is provided.
Given neighboring databases $D_0, D_1$, a randomized mechanism $\mc M$ and output $Y$, consider the following hypothesis testing problem: 
\begin{align*}
    H_0 \quad&:\quad Y \text{ came from database } D_0, \\
    H_1 \quad&:\quad Y \text{ came from database } D_1.
\end{align*}
For a rejection region $S$ the probability of false alarm (type I error), i.e. the null hypothesis is true but rejected, is defined as $P_{FA}(D_0, D_1, \mc M, S)=\P(\mc M(D_0)\in S)$, and the probability of missed detection (type II error), i.e. the null hypothesis is false but retained, is $P_{MD}(D_0, D_1, \mc M, S)=\P(\mc M(D_1)\in S)$.
\begin{theorem}{(Theorem 2.1 from \cite{kairouz2015composition})}
For any $\epsilon\geq 0$ and $\delta\in[0,1]$, a database mechanism $\mathcal{M}$ is $(\epsilon, \delta)$-DP if and only if the following conditions are satisfied 
for all pairs of neighboring databases $D_0, D_1$, and all rejection regions $S\subset \mc X$:
\begin{align*}
    P_{FA}(D_0, D_1, \mc M, S) + e^{\epsilon}P_{MD}(D_0, D_1, \mc M, S)&\geq 1-\delta, \\
    e^{\epsilon}P_{FA}(D_0, D_1, \mc M, S) + P_{MD}(D_0, D_1, \mc M, S)&\geq 1-\delta.
\end{align*}
\end{theorem}
The hardness of this hypothesis test can be characterized by the optimal trade-off between type I and type II errors. Let $P,Q$ be the probability distributions $\mc M(D_0)$, rep. $\mc M(D_1)$.
For a rejection rule $0\leq \phi\leq 1$, define the type I and type II error rates as
\begin{equation*}
    \alpha_\phi = \E_P[\phi], \quad \beta_\phi = 1-\E_Q[\phi].
\end{equation*}
\begin{definition}{(trade-off function)}
For any two probability distributions $P$ and $Q$ on the same space, the trade-off function $T(P,Q)\rightarrow [0,1]$ is defined as
\begin{equation*}
    T(P,Q)(\alpha) = \inf\{\beta_\phi : \alpha_\phi\leq \alpha\},
\end{equation*}
where the infimum is taken over all measurable rejection rules $\phi$.
\end{definition}

The trade-off function defines the boundary between achievable and unachievable regions of type I and type II errors of the hypothesis test. This leads to the definition of $f$ differential privacy. Define $f\geq g$, if $f(x)\geq g(x)$ for all $x\in[0,1]$.
\begin{definition}{($f$-DP)}
Let $f$ be a trade-off function. A mechanism $\mc M$ is said to be $f$-differentially private, if 
\begin{equation*}
    T\left(\mc M(D_0), \mc M(D_1)\right) \geq f
\end{equation*}
for all neighboring datasets $D_0, D_1$.
\end{definition}
The connection to $(\epsilon, \delta)$-DP is as follows:
For given $(\epsilon, \delta)$, define the trade-off function
\begin{equation*}
    f_{\epsilon, \delta}(\alpha) = \max\{0, 1-\delta-e^{\epsilon}\alpha, e^{-\epsilon}(1-\delta-\alpha)\}.
\end{equation*}

Furthermore, in the $f$-DP framework, privacy guarantees can be compared to the trade-off for distinguishing between normal distributions with unit variance and differing means. This notion of DP, called Gaussian-DP (GDP), introduced in \cite{dong2019gaussian}, not only gives tight composition bounds, but also enjoys a central limit-type theorem that states that the composition of differentially private mechanisms in the limit enjoys a GDP trade-off. 

\subsection{Differential Privacy for Generative Models}

A generative model with parameters $\theta$ learns a probability density $p_\theta(x)$ on data space $\mc X$, from which new data points can be sampled. The parameters $\theta$ are adjusted to increase the likelihood of training points $x\in D\subset \mc X$. The definition of $(\epsilon, \delta)$-DP thus gives a bound on how much the likelihood of a subset $S\subset \mc X$ can change by modifying one training point.

\begin{definition}{($(\epsilon, \delta)$-DP for generative models)}
\label{def:dp_gm}
A generative model $p_\theta$ is DP for $\epsilon\geq 0$, $\delta\geq 0$ if,
    for neighboring databases $D_0, D_1$ and corresponding trained models $p_{\theta(D_0)}, p_{\theta(D_1)}$ and any $S\subset \mc X$ it holds
    \begin{equation}
        \int_S p_{\theta(D_0)} dx \leq e^\epsilon \int_S p_{\theta(D_1)} dx + \delta.
    \end{equation}
\end{definition}

\subsection{Membership Inference Attacks}

Membership inference \cite{shokri2017membership, salem2018ml} is an attack on a machine learning model, where an adversary tries to infer whether a particular record  was part of the training set of the attacked model . Depending on the type of ML model and the information about that model that is available to the adversary, there exist different variants of the membership inference attack (MIA). 
The types of MIAs that can be utilized against generative models depend on the knowledge of the attacker. In a white-box setting the attacker can access the parameters of the generative model itself and use this knowledge to craft an attack. Here we focus on a block-box setting, i.e. the model parameters of the target model are not available.
In \cite{shokri2017membership} it is assumed that the attacker has access to public data that is drawn from the same distribution as the target model's training data, and the architecture and parameters of the target model are known. Several so-called shadow models are trained that mimic the behaviour of the target model. An attack model is then trained on the shadow models' outputs and used to classify the record under consideration as member or non-member of the target model's training data.
Later works were able to relax these assumptions \cite{salem2018ml, hilprecht2019monte}.The most common attacks against generative models/synthetic data in a black box setting are as follows:
\begin{enumerate}
    \item Classifier trained on shadow models \cite{shokri2017membership}:
    The MIA then proceeds with the following steps:
    \begin{enumerate}
        \item Draw $m$ datasets of size $n$ from the public data.
        \item Train $2m$ shadow models, one for each data set with and without inserting the target record $t$ into the set.
        \item Generate $p$ synthetic data sets from each shadow model. This results in a labelled data set of size $2mp$ of pairs $(x,y)$, where $x$ is a synthetic data set and $y=0$ if the target record was not in the set used to train the model that generated $x$, $y=1$, if the target record was present.
        \item Train a (binary) attack classifier on the labelled set of data sets.
        \item Use the attack classifier to predict, whether the target record was present in the private data that was used to generate the synthetic data set.
    \end{enumerate}
    The drawback of this method is that it is computationally very expensive. For each target record, several shadow models have to be trained and the resulting attack classifier is only usable for this specific record.
    \item GAN-based attack \cite{hayes2017logan}: \\
    Here, a GAN is trained on synthetic data generated by the target model. Its discriminator's confidence is used to predict whether a target record was part of the training of the target model. \\
    In contrast to the attack described above, only one model has to be trained, which can be used for all target records of interest.
    \item Monte Carlo attack \cite{hilprecht2019monte}: \\
    The simplest MIA on synthetic data does not need any shadow model at all. The classification of target records into members and non-members of the target model's training set is based on thresholding the distance of a target record to its nearest synthetic records in a suitable norm. Despite its apparent simplicity this attack proved surprisingly effective, especially in cases where the generative model overfitted to the training set.
\end{enumerate}

Clearly, for a $f$-DP model $\mc M$, the trade-off between type I and type II errors of the resulting attack classifier is lower bounded by the trade-off function $f$.
For the empirical privacy evaluation in this paper the Monte Carlo attack is employed.

\section{Differentially Private Generative Models}
\label{sec:vae}
Deep learning based generative models have become prevalent for the generation of synthetic data due to their capabilities of learning complex distributions from high dimensional data. The most prominent models in the literature are generative adversarial networks (GAN), variational autoencoders and variants thereof, and diffusion models. While GANs tend to produce higher quality samples, VAEs are more open to theoretical analysis since the data distribution is modeled explicitly, following the Bayesian paradigm. In this work, a VAE architecture is modified to satisfy DP as in Definition \ref{def:dp_gm}.

\subsection{Variational Autoencoders}
In VAEs, data is generated by drawing a latent variable $z\in \R^d$ from a prior distribution $p(z)$, which is a simple distribution, e.g. a Gaussian with zero mean and identity covariance matrix. A sample is then drawn according to a parameterized likelihood $x\sim p_\theta(x|z)$. Hence, the data distribution is approximated as 
\begin{equation*}
    p_\theta(x) = \int p_\theta(x|z)p(z)dz.
\end{equation*}
The parameters of the likelihood are fit by maximizing this marginal likelihood $p_\theta(x)$ of the data points $\{x_1,\ldots, x_N\}\subset \mc X$. This is achieved by maximizing the so-called evidence lower bound (ELBO), i.e.
\begin{align}
    &\log p_\theta(x) = \log q_\phi(z|x)\frac{p_\theta(x,z)}{q_\phi(z|x)}dz \nonumber \\
    \label{eq:elbo}
    \geq &\E_{q_\phi(z|x)}\log \frac{p_\theta(x,z)}{q_\phi(z|x)} = \mc L(x; \phi,\theta),
\end{align}
where $q_\phi(z|x)$ is a variational density, usually a parameterized Gaussian, that approximates the true posterior distribution $p_\theta(z|x)=p_\theta(x|z)p(z)/p_\theta(x)$, which is intractable due to the integration in the denominator. In practice, the likelihood $p_\theta(x|z)$ is parameterized by a decoder network that outputs the parameters of the distribution of $x$, and similarly the approximate posterior $q_\phi(z|x)$ by an encoder network that, given a data point $x$, outputs the parameters for a Gaussian distribution of $z|x$ in latent space.

Several ways of interpreting the ELBO objective have been explored in the literature (e.g. \cite{hoffman2016elbo} and references therein), emphasizing different aspects of this variational framework, and thus gave rise to various improvement strategies.
The formulation
\begin{equation*}
    \mc L(x;\phi,\theta) = p_\theta(x) -\KL\left(q_\phi(z|x)||p_\theta(z|x)\right)
\end{equation*}
shows that the ELBO can be tightened by either improving the prior $p(z)$ or the likelihood $p_\theta(x|z)$, thus improving the log evidence of the model, or by reducing posterior mismatch, as expressed by the KL term. The latter can be achieved by allowing for more expressive approximate posteriors than mere Gaussians with diagonal covariance matrix. Arguably the most common formulation of the ELBO, and usually the basis for implementation, is as average reconstruction error minus KL regulariser,
\begin{align*}
    \E_{x\sim p(x)}\left[\mc L(x;\phi,\theta]\right] = \frac1N\sum\limits_{i=1}^N[&\E_{z_i\sim q_\phi(z|x_i)}\left[\log p_\theta(x_i|z_i)\right] \\
    &- \KL(q_\phi(z_i|x_i)||p(z))].
\end{align*}
Achieving disentangled latent representations lead to the introduction of the $\beta$-VAE by reweighting the two parts of the objective \cite{higgins2017beta}. A large body of literature focusses on learnable priors (e.g. \cite{tomczak2018vae}), thus facilitating to match approximate posterior and prior, and improving sample quality by preventing high density prior regions with low density posterior. 

In \cite{hoffman2016elbo}, another interpretation of the ELBO objective is derived, which further decomposes the prior matching term to shed light on the possible causes of ELBO sub-optimality.

To this end, the following 'index code' terms are defined, which arise from uniformly sampling from the training data points $x_1,\ldots, x_N \in D$. 
\begin{align}
    q(i,z) &= q(z|i)q(i), \quad q(z|i)=q_\phi(z|x_i), \quad q(i)=\frac1N, \\
    p(i,z) &= p(z|i)p(i), \quad p(z|i)=p(z), \quad p(i) = \frac1N.
\end{align}
Furthermore, define the random variables $\mathbf Z\sim p(z)$ and $\mathbf N\sim \frac1N$ and the average encoding distribution 
\begin{equation}
  q_\phi^{avg}(z) = \frac1N \sum\limits_{i=1}^N q_\phi(z|x_i).  
\end{equation}
With these terms, the KL-term is decomposed as follows.

\begin{lemma}{(Decomposition from \cite{hoffman2016elbo})}
\label{lemma:idx_code}
    Let $x_1, \ldots x_N\sim p(x)$. Then 
    \begin{equation}
    \begin{split}
        \frac1N \sum\limits_{i=1}^N \KL(q_\phi(z|x_i)||p(z)) = 
        &\KL(q_\phi^{avg}(z)||p(z)) \\
        + &\log N - \H(\mathbf N|\mathbf Z).
        \end{split}
    \end{equation}
\end{lemma}


Empirically it is confirmed that a VAE in standard operation  produces encodings with very small variance in order to facilitate reconstruction. The average encoding distribution can closely match the prior nevertheless. This precise encoding leads to the conditional entropy term being low, which in terms of privacy implies that the encoder is not private, since latent representations can be mapped back unambiguously to their respective pre-images. An encoder operating under a privacy guarantee on the other hand prevents this one-to-one mapping of latent representations and corresponding inputs, thus the entropy term is close to its maximum $\log N$. Under the premise that the average encoding distribution is still close to the prior, this implies that also the encodings of single data points match the prior more closely, i.e. $\KL(q_\phi(z|x_i)||p(z))$ is small for all $i=1,\ldots N$.

\subsection{Lipschitz-Regularised VAE}
In this work, the VAE architecture is modified to achieve DP. Therefore, the decoder is constrained to learning Lipschitz continuous mappings with pre-specified Lipschitz constant.
The Lipschitz constant of a function $f:X\rightarrow Y$ between metric spaces $X$ and $Y$ is defined as the smallest constant $L>0$ such that
\begin{equation*}
    \forall x,y\in X: \|f(x)-f(y)\|_Y \leq L \|x-y\|_X,
\end{equation*}
where $\|\cdot\|_X$ and $\|\cdot\|_Y$ denote the norms of the metric spaces $X$ and $Y$.
A Lipschitz continuous function thus can not map points that are close arbitrarily far apart. In particular, points that are mapped close together by the encoder of the VAE will have similar reconstructions. This in turn has implications for data points that deviate significantly from typical entries in the training data, i.e. outliers. These deviant training samples can either be mapped by the encoder into a high density region of the latent prior $p(z)$ (i.e. around 0), where, due to the Lipschitz constraint on the decoder they can not be reconstructed properly, or, in order to keep the reconstruction error small while retaining a smooth latent embedding, have to be mapped into a low density region far away from the origin in latent space. In both cases, the outliers are unlikely to be reproduced in the generation process, since latent space samples close to the latent embeddings of these outliers are either not properly reconstructed, or very unlikely to be sampled at all.

To make the analysis tractable we consider the case of a VAE with mean squared error reconstruction loss, i.e. the likelihood of the data is parameterized as a multivariate Gaussian with unit covariance matrix and learned mean $\mu_\theta$:
\begin{equation}
\label{eq:likelihood}
    p_\theta(x|z) = \frac{1}{\sqrt{(2\pi)^n}} \exp\left(-\frac12\|x-\mu_\theta(z)\|^2 \right).
\end{equation}

\begin{assumption}[Lipschitz continuous decoder]
    The decoder $\mu_\theta : \mc Z\rightarrow \mc X$ is $L>0$ Lipschitz continuous, i.e. for any weights $\theta$ it holds
    \begin{equation}
        \label{eq:lip_enc}
        \forall z_1, z_2 \in\mc Z: \|\mu_\theta(z_1)-\mu_\theta(z_2)\| \leq L \|z_1-z_2\|.
    \end{equation}
\end{assumption}

The prior over the latent space $p(z)$ is chosen as a $d$-dimensional standard normal distribution. Although this implies an unbounded space $\mc Z=\R^d$, any randomly drawn sample $z\sim p(z)$ falls within a ball of radius $R_z$ with high probability. We define $R_z(\delta)$ as the radius of this ball such that, for given $\delta>0$
\begin{equation}
    \label{eq:radius}
    \P\left(\|Z\|\leq R_z(\delta)\right)>1-\delta, \qquad Z\sim \mc N(\mathbf{0}, \mathbf{I}_d).
\end{equation}
We can assume that the sample space is contained in a ball around 0 with radius $R_x$, $\mc X\subset \{x\in\R^n : \|x\|\leq R_x\}$. $R_x$ can be controlled by rescaling the data.

For our analysis we will assume sufficient model power (especially in the encoder) such that any trained encoder/decoder pair satisfies a consistency property:
\begin{assumption}[Consistency]
    For subsets $B_X\subset\mc X$, $B_Z\subset \mc Z$ it holds:
    \begin{equation}
        \label{eq:consistency}
        \forall x\in B_X, \forall z\in B_Z: q_\phi(z|x)p(x) = p_\theta(x|z)p(z).
    \end{equation}
\end{assumption}

Our modified VAE setup can thus be summarised as the constrained optimisation problem
\begin{equation}
\begin{split}
    &\min\limits_{\phi,\theta} -\frac1N\sum\limits_{i=1}^N\mc L(x_i; \phi,\theta) \\ &\text{subject to $\mu_\theta$ is $L$-Lipschitz continuous.}
    \end{split}
\end{equation}

\section{Privacy Analysis of L-VAE}
\label{sec:privacy_analysis}
The privacy preserving properties of the Lipschitz constrained VAE are analyzed theoretically using tools from information theory, an information theoretic characterization of differential privacy \cite{cuff2016differential, wang2016relation} and results on privacy of posterior sampling from \cite{wang2015privacy}. 

\subsection{Bounded log-Likelihood}
\begin{lemma}[Bounded log-likelihood]
\label{lemma:bounded_log_likelihood}
For all $x_1, x_2\in\mc X$ and all $z\in\mc Z$ it holds
\begin{equation}
    |\log p_\theta(x_1|z)-\log p_\theta(x_2|z)|\leq C.
\end{equation}
\end{lemma}
\begin{proof}
\begin{align*}
    &|\log p_\theta(x_1|z)-\log p_\theta(x_2|z)| \\
    =&\left| \frac12\|x_1-\mu_\theta(z)\|^2 - \frac12\|x_2-\mu_\theta(z)\|^2 \right| \\
    = & \frac12 \left(\|x_1-\mu_\theta(z)\| + \|x_2-\mu_\theta(z)\|\right) \\
    &\cdot\left( \left|\|x_1-\mu_\theta(z)\| - \|x_2-\mu_\theta(z)\| \right| \right) \\
    \leq &\frac12 \left(\|x_1-\mu_\theta(z)\| + \|x_2-\mu_\theta(z)\|\right)\|x_1-x_2\| \\
    =& \frac12(\|x_1-\mu_\theta(z_1) + \mu_\theta(z_1)-\mu_\theta(z)\| \\ 
    &+ \|x_2-\mu_\theta(z_2) + \mu_\theta(z_2)-\mu_\theta(z)\| )\|x_1-x_2\| \\
    \leq & (R_x+2LR_Z(\delta))\|x_1-x_2\|\leq R_x^2 + 2LR_Z(\delta)R_x
\end{align*}
If $x_1 = \mu_\theta(z_1)$ and $x_2 = \mu_\theta(z_2)$, this bound becomes 
\begin{align*}
    |\log p_\theta(x_1|z)-\log p_\theta(x_2|z)| \leq 4L^2R^2_Z(\delta)
\end{align*}
\end{proof}
\subsection{Differential Privacy of the Encoder}
Sampling from the posterior of a Bayesian model with bounded log-likelihood is an example of the exponential mechanism and hence satisfies $\epsilon$-DP. By our consistency assumption, this holds true for the encoder of the L-VAE as well.

\begin{lemma}[Posterior lemma]
    Let $p_\theta(x|z)$ satisfy the assumptions, and $q_\phi(z|x)$ be the corresponding encoder, linked to $p_\theta(x|z)$ via consistency, i.e.
    \begin{equation}
    \label{eq:consistency_post}
    q_\phi(z|x)p_\theta(x)=p_\theta(z|x)p(z). 
    \end{equation}
    Then $q_\phi(z|x)$ is $\epsilon$-DP with $\epsilon=2C$, where $C$ is the constant from lemma \ref{lemma:bounded_log_likelihood}.
\end{lemma}
\begin{proof}
    Let $B\subset \mc Z$ and $x_1,x_2\in \mc X$. Then
    \begin{align*}
        \int_Bq_\phi(z|x_1)dz & = \int_Bq_\phi(z|x_1)\frac{q_\phi(z|x_2)}{q_\phi(z|x_2)}dz \\
        & = \int_B\frac{q_\phi(z|x_1)}{q_\phi(z|x_2)} q_\phi(z|x_2)dz
    \end{align*}
    Using \eqref{eq:consistency_post} and $p_\theta(x)=\int_Z p_\theta(x|z)p(z)dz$, we have
    \begin{align*}
        &\frac{q_\phi(z|x_1)}{q_\phi(z|x_2)} = \frac{p_\theta(x_1|z)p(z)}{p_\theta(x_1)}\cdot\frac{p_\theta(x_2)}{p_\theta(x_2|z)p(z)} \\
        =& \frac{p_\theta(x_2)}{p_\theta(x_1)}\cdot\frac{p_\theta(x_1|z)}{p_\theta(x_2|z)} \\
        \leq &\frac{p_\theta(x_2)}{p_\theta(x_1)} \exp\left(\left| \log p_\theta(x_1|z)-\log p_\theta(x_2|z) \right|\right)  \\
        \leq &\frac{p_\theta(x_2)}{p_\theta(x_1)} \exp(C).
    \end{align*}
    The first factor can be bounded analogously as
    \begin{align*}
        &\frac{p_\theta(x_2)}{p_\theta(x_1)} = \frac{\int_Zp_\theta(x_2|z)p(z)dz}{\int_Zp_\theta(x_1|z)p(z)dz} \\
        =&\frac{\int_Zp_\theta(x_2|z)\frac{p_\theta(x_1|z)}{p_\theta(x_1|z)}p(z)dz}{\int_Zp_\theta(x_1|z)p(z)dz} = \frac{\int_Z\frac{p_\theta(x_2|z)}{p_\theta(x_1|z)}p_\theta(x_1|z)p(z)dz}{\int_Zp_\theta(x_1|z)p(z)dz} \\
        \leq & \frac{\int_Z \exp\left(\left| \log p_\theta(x_2|z)-\log p_\theta(x_1|z) \right| \right)p_\theta(x_1|z)p(z)dz}{\int_Zp_\theta(x_1|z)p(z)dz} \\
        \leq &\exp(C).
    \end{align*}
    Hence, we have
    \begin{align*}
        \int_Bq_\phi(z|x_1)dz & \leq \exp(2C) \int_B q_\phi(z|x_2)dz.
    \end{align*}
\end{proof}
\subsection{Differentially Private Synthetic Data Generation}
The previous lemma established for points $x_1, x_2$ and two sets of weights $\phi_1,\phi_2$ that
$q_{\phi_1}(z|x_1) \leq e^{\epsilon} q_{\phi_1}(z|x_2)$ and
$q_{\phi_2}(z|x_1)\leq e^{\epsilon} q_{\phi_2}(z|x_2)$. We need to take the variation of the model parameters into account as well. 
\begin{lemma}
\label{lemma:enc_weights}
Let $D_1, D_2$ be neighboring datasets and denote the $\epsilon$-DP encoders of a L-VAE trained on these datasets by $q_{\phi_1}(z|x)$ and $q_{\phi_2}(z|x)$. Then, for $B\subset B_Z(\delta)$ and $N$ points $x_i\sim p(x)$, $x_i\in A\subset \mc X$, it holds
\begin{equation}
\int_B\frac1N \sum_{i} \left[q_{\phi_1}(z|x_i)-q_{\phi_2}(z|x_i)\right] dz \leq 2\sqrt{\epsilon}|B|.
\end{equation}
\end{lemma} 

\begin{proof}[proof of lemma \ref{lemma:enc_weights}]
    \begin{align*}
        \frac{1}{|B|}&\int_B \frac1N \sum\limits_{i=1}^N\left[q_{\phi_1}(z|x_i)-q_{\phi_2}(z|x_i) \right]dz \\
        &\leq \sum\limits_{i=1}^N \sup\limits_{C\subseteq B} \left| q_{\phi_1}(C|x_i)-q_{\phi_2}(C|x_i) \right|  \\
        &=\sum\limits_{i=1}^N \sup\limits_{C\subseteq B} \left| q_{\phi_1}(C|x_i) - p(C) + p(C)-q_{\phi_2}(C|x_i) \right| \\
        &\leq \sum\limits_{i=1}^N \sup\limits_{C\subseteq B} \left| q_{\phi_1}(C|x_i)-p(C) \right| 
        + \sum\limits_{i=1}^N \sup\limits_{C\subseteq B} \left| q_{\phi_2}(C|x_i)-p(C) \right|   \\
        &\leq \sum\limits_{i=1}^N\sqrt{\KL(q_{\phi_1}(z|x_i)||p(z))} +\sum\limits_{i=1}^N\sqrt{\KL(q_{\phi_2}(z|x_i)||p(z))} \\
        &\leq \sqrt{\sum\limits_{i=1}^N\KL(q_{\phi_1}(z|x_i)||p(z))} 
        +\sqrt{\sum\limits_{i=1}^N\KL(q_{\phi_2}(z|x_i)||p(z))} \\
        &=\sqrt{\I_{q_{\phi_1}(i,z)}(N,Z)} + \sqrt{\I_{q_{\phi_2}(i,z)}(N,Z)}. 
    \end{align*}
In the last line, Lemma \ref{lemma:idx_code} was used.
Since $q_{\phi_1}(z|x)$ and $q_{\phi_2}(z|x)$ are each $\epsilon$-DP mechanisms the index code mutual information terms can be upper bounded by $\epsilon$, due to the equivalence of $\epsilon$-DP and MI-DP (\cite{cuff2016differential}, thm. 1 and corr. 1). Therefore,
\begin{align}
&\frac{1}{|B|}\int_B \frac1N \sum\limits_{i=1}^N\left[q_{\phi_1}(z|x_i)-q_{\phi_2}(z|x_i) \right]dz \\
&\leq \sqrt{\I_{q_{\phi_1}(i,z)}(n,z)} + \sqrt{\I_{q_{\phi_2}(i,z)}(n,z)} \\
& \leq \sqrt{\epsilon} + \sqrt{\epsilon}\leq 2\sqrt{\epsilon}.    
\end{align}
\end{proof}

Finally, the main theorem about the privacy of generated data can be proved.

\begin{proof}
Let $A\subset B_X$.
    \begin{align*}
        &\int_A p_\theta(x)dx = \int_A\int_{\mc Z} p_\theta(x|z)p(z) dz dx \\
        =&\int_A\int_{B_Z} q_\phi(z|x)p(x) dz dx + \int_A\int_{B_Z^C} p_\theta(x|z)p(z)dz dx
    \end{align*}
    because of the assumed consistency \eqref{eq:consistency} inside $B_Z\times B_X$. Now,
    by definition of $B_Z$
    \begin{align*}
      &\int_A\int_{B_Z} q_\phi(z|x)p(x) dz dx + \int_A\int_{B_Z^C} p_\theta(x|z)p(z)dz dx \\
       = &\int_{B_Z} \int_A q_\phi(z|x)p(x) dx dz + \int_{B_Z^C} \int_A p_\theta(x|z)dx p(z)dz \\
      \leq & \int_{B_Z} \int_A q_\phi(z|x)p(x) dz dx + \delta \\
      = &\int_{B_Z} |A| \frac1N \sum\limits_{i=1}^N q_\phi(z|x_i) dz + \delta + \delta_x 
    \end{align*}
   Here, the integral over $A$ is approximated via Monte-Carlo integration. The points $\{x_1,\ldots, x_N\}\subset A$ are drawn independently from $p(x)$. The error term $\delta_x$ decreases as $1/\sqrt{N}$. Lemma \ref{lemma:enc_weights} is applied.
   \begin{align*}
     &\int_{B_Z} |A| \frac1N \sum\limits_{i=1}^N q_\phi(z|x_i) dz + \delta + \delta_x \\
     =&\int_{B_Z} |A| \frac1N \sum\limits_{i=1}^N q_{\Tilde{\phi}}(z|x_i) dz \\
     &+ \int_{B_Z} |A| \frac1N \sum\limits_{i=1}^N \left[q_\phi(z|x_i)-q_{\Tilde{\phi}}(z|x_i) 
     \right]dz + \delta + \delta_x \\
     \leq & \int_{B_Z} |A| \frac1N \sum\limits_{i=1}^N q_{\Tilde{\phi}}(z|x_i) dz + \delta + \delta_x + 2\sqrt{\epsilon}|B|\frac{1}{|A|}. 
   \end{align*}
   What remains, is to go through the steps in reverse order to achieve a bound in terms of $p_{\Tilde{\theta}}(x|z)$ to complete the proof. Therefore, with an error term $\Tilde{\delta}_x$ that goes to zeros as $1/\sqrt{N}$, 
   \begin{align*}
       &\int_{B_Z} |A| \frac1N \sum\limits_{i=1}^N q_{\Tilde{\phi}}(z|x_i) dz + \delta + \delta_x + 2\sqrt{\epsilon}|B|\frac{1}{|A|} \\
       \leq &\int_{B_Z} \int_A q_{\tilde{\phi}}(z|x)p(x) dx dz + \delta + \delta_x + \Tilde{\delta}_x + 2\sqrt{\epsilon}|B|\frac{1}{|A|}.
   \end{align*}
   Writing $\Bar{\delta} = \delta + \delta_x + \Tilde{\delta}_x + 2\sqrt{\epsilon}\frac{|B|}{|A|}$ for convenience, and using again the consistency assumption,
   \begin{align*}
       &\int_{B_Z} \int_A q_{\tilde{\phi}}(z|x)p(x) dx dz + \Bar{\delta} \\
       \leq & \int_{B_Z} \int_A p_{\tilde{\theta}}(x|z)p(z) dx dz \\
       &+ 
       \int_{B_Z^C} \int_A p_{\tilde{\theta}}(x|z)p(z) dx dz + \Bar{\delta} \\
       = & \int_{\mc Z} \int_A p_{\tilde{\theta}}(x|z)p(z) dx dz + \Bar{\delta}
       =  \int_A p_{\tilde{\theta}}(x) dx  + \Bar{\delta}
   \end{align*}
\end{proof}

\section{Experiments}
\label{sec:experiments}
The proposed method is evaluated with respect to the quality of the generated data and its capacity to preserve the privacy of the training data on the popular MNIST dataset \cite{deng2012mnist}. The L-VAE is compared to a VAE with the same encoder and decoder architectures, trained 1.) without any privacy protection measures, and 2.) using a differentially private optimiser as in \cite{abadi2016deep}.  While the findings in  \cite{stadler2021synthetic} suggest that merely generating synthetic data is not enough to properly protect privacy, their results nevertheless show that it is hard to infer membership for typical data, and it is particularly outliers that are exposed to privacy attacks. Hence, we insert a white image (all pixels have value 1) into the training set and observe its reconstruction with particular attention.

\subsection{Implementation Details}
All experiments were conducted in python, using TensorFlow 2.6 \footnote{\url{https://www.tensorflow.org/}}
for the implementation of the VAE variants and scikit-learn\footnote{\url{https://scikit-learn.org/}} for the implementation of the Monte-Carlo MIA.
Two variants of the L-VAE are examined: The Lipschitz property of the decoder can be obtained by adding a gradient penalty to the ELBO objective, in the same way as is done in improved training for Wasserstein GANs \cite{gulrajani2017improved}. Since the gradient of the decoder with respect to its input variable $z$ is only probed at a finite number of points, this method cannot strictly guarantee that the Lipschitz constraint holds everywhere. Moreover, depending on the weighting of the penalty term, it may happen that the constraint is not satisfied exactly, which then leads to a larger Lipschitz constant. A more rigorous way is to use spectral normalization \cite{miyato2018spectral} on each of the decoder's layers. Spectral normalization estimates the spectral norm of a layer (pre activation) by a few steps of a power iteration algorithm and divides the weight matrix by the estimated spectral norm. Hence, the layer becomes 1-Lipschitz. Most activation functions such as ReLU or sigmoid activations are also 1-Lipschitz. To guarantee that the whole decoder is $L$-Lipschitz for a chosen constant $L>0$, it is thus sufficient to apply spectral normalisation to each layer and multiply its output by $L^{\frac1n}$, where $n$ is the number of layers. With spectral normalisation applied to the decoder the Lipschitz constraint is ensured to always hold. In the experiments on the MNIST dataset the gradient penalty version worked better than the alternative and hence in Figure \ref{fig:mnist_vae_gp} these results are reported.

\underline{\textbf{MNIST}} 
For the experiments on MNIST, the encoder and decoder of all VAE variants consist of convolution resp. transpose convolution layers and dense layers with ReLU activations, except for the output layers that do not apply any activation function. The encoder consists of 4 convolution layers with 32, 64, 128 and 32 filters of size $3\times 3$ and strides 2, followed by two dense layers with output sizes 100 and $2\cdot n_{latent}$, where $n_{latent}$ denotes the dimension of the latent space. The decoder takes samples from the latent space as input, which go through two dense layers of dimension 100 and $7\cdot 7\cdot 32$ resp., followed by transpose convolution layers with 64, 128 and 1 filters of size $3\times 3$ with strides 2. The MNIST samples of dimensions $28\times28\times 1$ are binarised to have values in $\{0,1\}$, and hence, binary crossentropy is used as loss function for the reconstruction. 9999 samples from the MNIST training set are taken uniformly at random plus the all-white outlier image, yielding 10k training samples. For the privacy evaluation 1000 non-members are sampled from the MNIST test set and 1000 members from the training set. For each of the resulting 2000 images the Monte-Carlo attack is conducted. The resulting trade-off between type I and type II errors for varying thresholds are reported below.\\


\subsection{Results}

\begin{figure}[ptb]
    \centering
    \includegraphics[width=.485\textwidth]{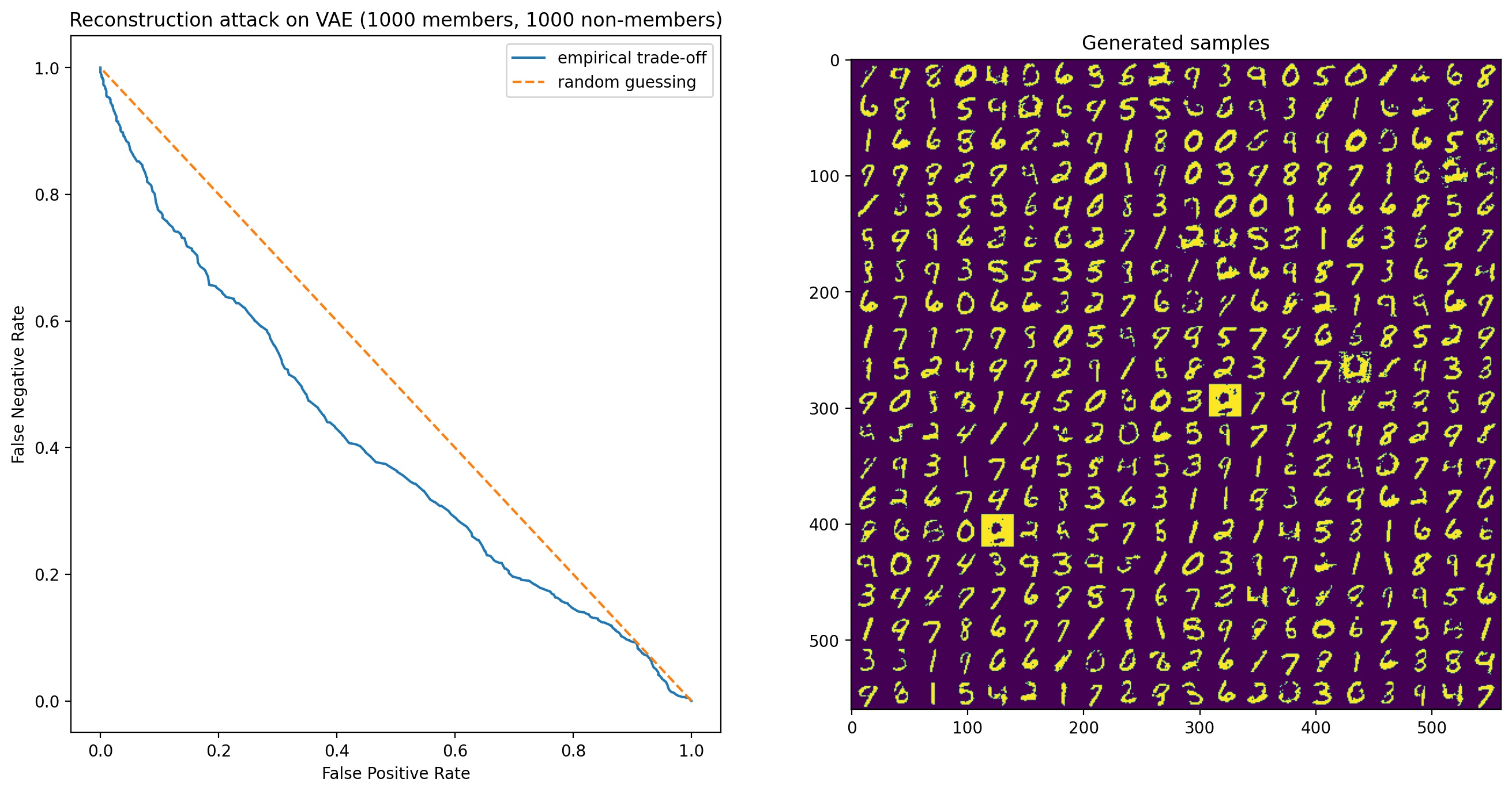}
    \caption{VAE - attack trade-off curve and generated samples}
    \label{fig:mnist_vae}
\end{figure}

\begin{figure}
    \centering
    \includegraphics[width=.485\textwidth]{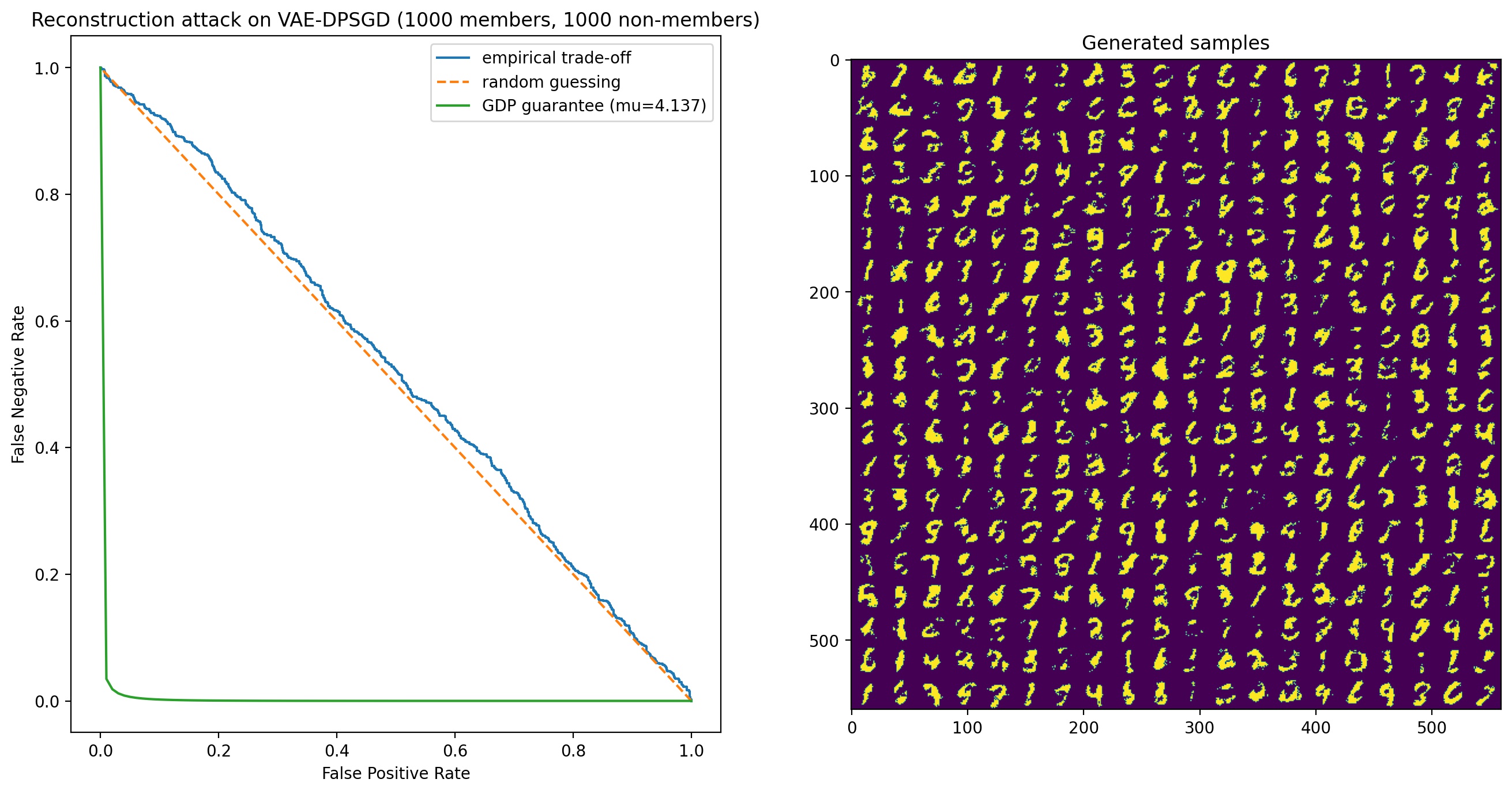}
    \caption{VAE-DPSGD - attack trade-off curve and generated samples}
    \label{fig:mnist_vae_dpsgd}
\end{figure}

\begin{figure}
    \centering
    \includegraphics[width=.485\textwidth]{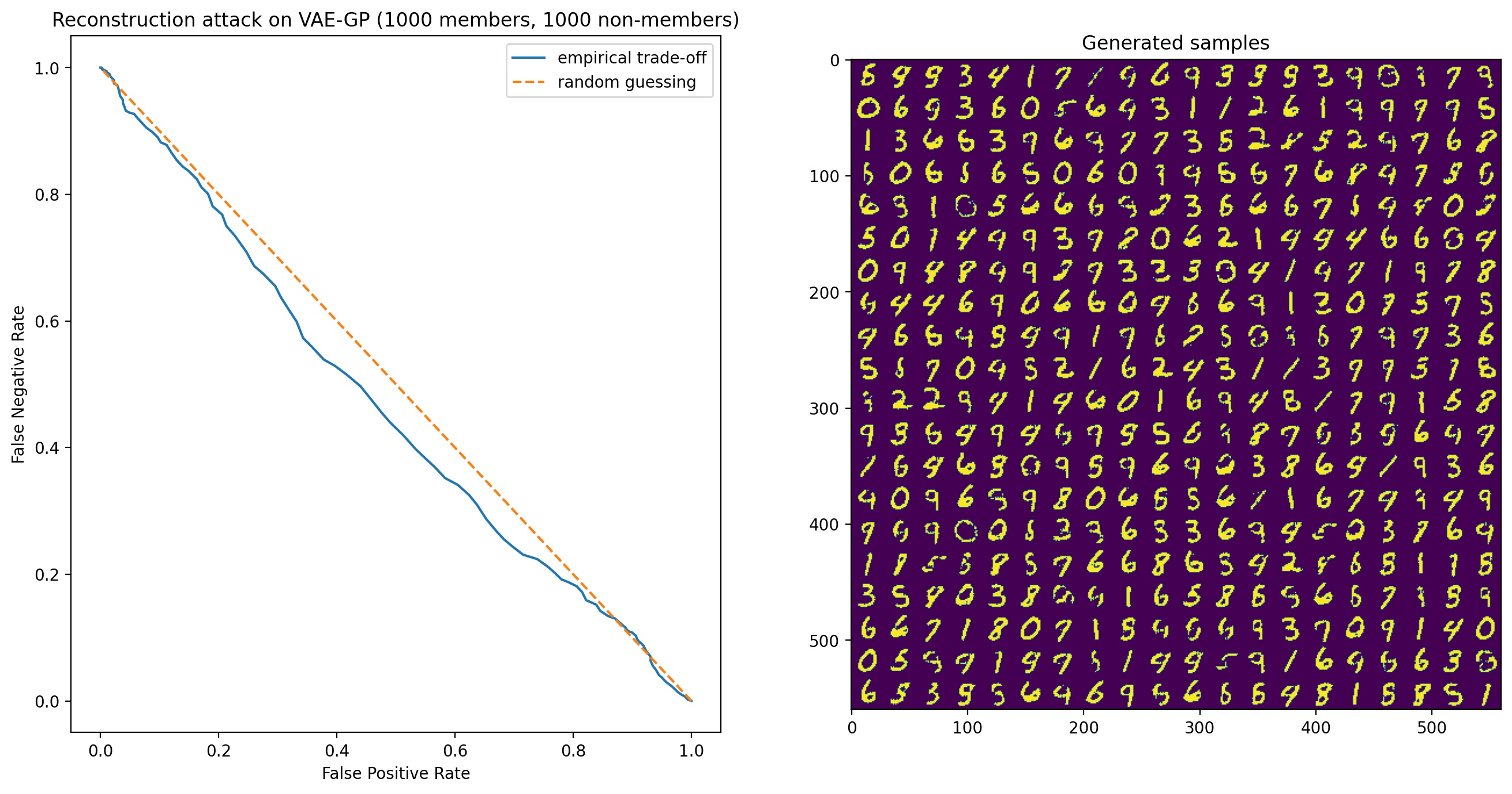}
    \caption{LVAE - attack trade-off curve and generated samples}
    \label{fig:mnist_vae_gp}
\end{figure}

In Figures \ref{fig:mnist_vae}, \ref{fig:mnist_vae_dpsgd}, \ref{fig:mnist_vae_gp} the resulting trade-off curves of the conducted MIA and synthetic samples are shown for the three VAE variants considered. While DP-SGD is the most effective for preventing membership inference, despite its loose theoretical guarantee (depicted in green), the sample quality is much worse than the proposed approach. The outlier (yellow square) in the training data is clearly reconstructed in the vanilla VAE, while both privacy preserving variants do not reconstruct this particular sample, and hence prevent it from occuring in the generated data.

\section{Conclusion and Upcoming Work}
\label{sec:conclusion}

Constraining the Lipschitz constant of the decoder network in a variational autoencoder enhances the privacy protecting capabilities of synthetic data generation, while allowing for higher quality samples than conventional differential privacy techniques. A privacy guarantee that is independent of the number of training epochs is crucial for training useful differentially private deep generative models. Training generative models with DP-SGD
often results in loose privacy bounds and/or bad sample quality, especially in the absence of public data on which the models can be pre-trained. PATE-GAN \cite{jordon2018pate} reports much better results, but may not be applicable for small datasets. Both these methods rely on the Gaussian mechanism and thus might inject more noise into the training than necessary. The proposed method in contrast constrains the mapping that is learned by the decoder and achieves privacy by exploiting the Bayesian nature of the VAE.
While the empirical results are highly positive, the theory needs to be further expanded, and especially the connection between the chosen Lipschitz constant and the privacy guarantee needs to be quantified exactly. An expansion of the proposed technique towards more expressive architectures and in particular generative models for mixed cataegorical/continuous data types is a topic of ongoing work. Initial experiments on the UCI Adult dataset (not reported here) suggest also improved performance compared to a VAE trained with DP-SGD. The challenges of mixed data call for architectures beyond standard VAEs, such as VAEM \cite{ma2020vaem} for which the proposed method seems particularly suited.

\section*{Acknowledgement}
This work was partially funded by the German Federal Ministry of Education and Research (BMBF) under project FT-Chain (01DS21011) and TinyPART (16KIS1395K). The authors acknowledge the financial support by the Federal Ministry of Education and Research of Germany in the programme of “Souverän. Digital. Vernetzt.” Joint project 6G-RIC, project identification number 16KISK020K.

\ifCLASSOPTIONcaptionsoff
  \newpage
\fi



%

\bibliography{refs}

\begin{thebibliography}{10}

\bibitem{dwork2006differential}
Cynthia Dwork.
\newblock Differential privacy.
\newblock In {\em International Colloquium on Automata, Languages, and
  Programming}, pages 1--12. Springer, 2006.

\bibitem{stadler2021synthetic}
Theresa Stadler, Bristena Oprisanu, and Carmela Troncoso.
\newblock Synthetic data--anonymisation groundhog day.
\newblock {\em arXiv preprint arXiv:2011.07018}, 2021.

\bibitem{ho2021dp}
Stella Ho, Youyang Qu, Bruce Gu, Longxiang Gao, Jianxin Li, and Yong Xiang.
\newblock Dp-gan: Differentially private consecutive data publishing using
  generative adversarial nets.
\newblock {\em Journal of Network and Computer Applications}, 185:103066, 2021.

\bibitem{jordon2018pate}
James Jordon, Jinsung Yoon, and Mihaela Van Der~Schaar.
\newblock Pate-gan: Generating synthetic data with differential privacy
  guarantees.
\newblock In {\em International conference on learning representations}, 2018.

\bibitem{torkzadehmahani2019dp}
Reihaneh Torkzadehmahani, Peter Kairouz, and Benedict Paten.
\newblock Dp-cgan: Differentially private synthetic data and label generation.
\newblock In {\em Proceedings of the IEEE/CVF Conference on Computer Vision and
  Pattern Recognition Workshops}, pages 0--0, 2019.

\bibitem{long2021g}
Yunhui Long, Boxin Wang, Zhuolin Yang, Bhavya Kailkhura, Aston Zhang, Carl
  Gunter, and Bo~Li.
\newblock G-pate: Scalable differentially private data generator via private
  aggregation of teacher discriminators.
\newblock {\em Advances in Neural Information Processing Systems}, 34, 2021.

\bibitem{DBLP:conf/iclr/PapernotAEGT17}
Nicolas Papernot, Mart{\'{\i}}n Abadi, {\'{U}}lfar Erlingsson, Ian~J.
  Goodfellow, and Kunal Talwar.
\newblock Semi-supervised knowledge transfer for deep learning from private
  training data.
\newblock In {\em 5th International Conference on Learning Representations,
  {ICLR} 2017, Toulon, France, April 24-26, 2017, Conference Track
  Proceedings}. OpenReview.net, 2017.

\bibitem{abadi2016deep}
Martin Abadi, Andy Chu, Ian Goodfellow, H~Brendan McMahan, Ilya Mironov, Kunal
  Talwar, and Li~Zhang.
\newblock Deep learning with differential privacy.
\newblock In {\em Proceedings of the 2016 ACM SIGSAC conference on computer and
  communications security}, pages 308--318, 2016.

\bibitem{kingma2013auto}
Diederik~P Kingma and Max Welling.
\newblock Auto-encoding variational bayes.
\newblock {\em arXiv preprint arXiv:1312.6114}, 2013.

\bibitem{hoffman2016elbo}
Matthew~D Hoffman and Matthew~J Johnson.
\newblock Elbo surgery: yet another way to carve up the variational evidence
  lower bound.
\newblock In {\em Workshop in Advances in Approximate Bayesian Inference,
  NIPS}, volume~1, 2016.

\bibitem{wang2015privacy}
Yu-Xiang Wang, Stephen Fienberg, and Alex Smola.
\newblock Privacy for free: Posterior sampling and stochastic gradient monte
  carlo.
\newblock In {\em International Conference on Machine Learning}, pages
  2493--2502. PMLR, 2015.

\bibitem{wang2019subsampled}
Yu-Xiang Wang, Borja Balle, and Shiva~Prasad Kasiviswanathan.
\newblock Subsampled r{\'e}nyi differential privacy and analytical moments
  accountant.
\newblock In {\em The 22nd International Conference on Artificial Intelligence
  and Statistics}, pages 1226--1235. PMLR, 2019.

\bibitem{kairouz2015composition}
Peter Kairouz, Sewoong Oh, and Pramod Viswanath.
\newblock The composition theorem for differential privacy.
\newblock In {\em International conference on machine learning}, pages
  1376--1385. PMLR, 2015.

\bibitem{dong2019gaussian}
Jinshuo Dong, Aaron Roth, and Weijie~J Su.
\newblock Gaussian differential privacy.
\newblock {\em arXiv preprint arXiv:1905.02383}, 2019.

\bibitem{shokri2017membership}
Reza Shokri, Marco Stronati, Congzheng Song, and Vitaly Shmatikov.
\newblock Membership inference attacks against machine learning models.
\newblock In {\em 2017 IEEE symposium on security and privacy (SP)}, pages
  3--18. IEEE, 2017.

\bibitem{salem2018ml}
Ahmed Salem, Yang Zhang, Mathias Humbert, Pascal Berrang, Mario Fritz, and
  Michael Backes.
\newblock Ml-leaks: Model and data independent membership inference attacks and
  defenses on machine learning models.
\newblock {\em arXiv preprint arXiv:1806.01246}, 2018.

\bibitem{hilprecht2019monte}
Benjamin Hilprecht, Martin H{\"a}rterich, and Daniel Bernau.
\newblock Monte carlo and reconstruction membership inference attacks against
  generative models.
\newblock {\em Proc. Priv. Enhancing Technol.}, 2019(4):232--249, 2019.

\bibitem{hayes2017logan}
Jamie Hayes, Luca Melis, George Danezis, and Emiliano De~Cristofaro.
\newblock Logan: Membership inference attacks against generative models.
\newblock {\em arXiv preprint arXiv:1705.07663}, 2017.

\bibitem{higgins2017beta}
Irina Higgins, Loic Matthey, Arka Pal, Christopher Burgess, Xavier Glorot,
  Matthew Botvinick, Shakir Mohamed, and Alexander Lerchner.
\newblock beta-vae: Learning basic visual concepts with a constrained
  variational framework.
\newblock In {\em International conference on learning representations}, 2017.

\bibitem{tomczak2018vae}
Jakub Tomczak and Max Welling.
\newblock Vae with a vampprior.
\newblock In {\em International Conference on Artificial Intelligence and
  Statistics}, pages 1214--1223. PMLR, 2018.

\bibitem{cuff2016differential}
Paul Cuff and Lanqing Yu.
\newblock Differential privacy as a mutual information constraint.
\newblock In {\em Proceedings of the 2016 ACM SIGSAC Conference on Computer and
  Communications Security}, pages 43--54, 2016.

\bibitem{wang2016relation}
Weina Wang, Lei Ying, and Junshan Zhang.
\newblock On the relation between identifiability, differential privacy, and
  mutual-information privacy.
\newblock {\em IEEE Transactions on Information Theory}, 62(9):5018--5029,
  2016.

\bibitem{deng2012mnist}
Li~Deng.
\newblock The mnist database of handwritten digit images for machine learning
  research.
\newblock {\em IEEE Signal Processing Magazine}, 29(6):141--142, 2012.

\bibitem{gulrajani2017improved}
Ishaan Gulrajani, Faruk Ahmed, Martin Arjovsky, Vincent Dumoulin, and Aaron~C
  Courville.
\newblock Improved training of wasserstein gans.
\newblock {\em Advances in neural information processing systems}, 30, 2017.

\bibitem{miyato2018spectral}
Takeru Miyato, Toshiki Kataoka, Masanori Koyama, and Yuichi Yoshida.
\newblock Spectral normalization for generative adversarial networks.
\newblock {\em arXiv preprint arXiv:1802.05957}, 2018.

\bibitem{ma2020vaem}
Chao Ma, Sebastian Tschiatschek, Richard Turner, Jos{\'e}~Miguel
  Hern{\'a}ndez-Lobato, and Cheng Zhang.
\newblock Vaem: a deep generative model for heterogeneous mixed type data.
\newblock {\em Advances in Neural Information Processing Systems},
  33:11237--11247, 2020.

\end{thebibliography}
\bibliographystyle{unsrt}
\end{document}